\documentclass{article}
\usepackage{iclr2027_conference,times}

\usepackage{amsmath,amsfonts,bm}

\def\eqref#1{equation~\ref{#1}}

\def\1{\bm{1}}

\DeclareMathAlphabet{\mathsfit}{\encodingdefault}{\sfdefault}{m}{sl}
\SetMathAlphabet{\mathsfit}{bold}{\encodingdefault}{\sfdefault}{bx}{n}

\usepackage{hyperref}
\usepackage{url}
\usepackage{algorithm}
\usepackage{algpseudocode}
\usepackage{wrapfig}
\usepackage{booktabs}
\usepackage{multirow}
\usepackage{amsfonts}
\usepackage{amsmath}
\usepackage{amssymb}
\usepackage{graphicx}
\usepackage{microtype}
\usepackage{xcolor}
\usepackage{amsthm}
\usepackage{caption}

\usepackage{amsthm}

\newtheorem{theorem}{Theorem}

\title{Federated Graph Retrieval-Augmented Generation}

\author{%
  \textbf{Xiaobao Wang}$^{1}$ \quad
  \textbf{Renda Han}$^{1}$ \quad
  \textbf{Zhiliang Xia}$^{1}$ \quad
  \textbf{Meng Liu}$^{2}$ \quad
  \textbf{Renxiang Guan}$^{3}$ \quad
  \textbf{Di Jin}$^{1}$\thanks{Corresponding author.} \quad\\
  $^{1}$College of Intelligence and Computing, Tianjin University, Tianjin, China \\
  $^{2}$Henan University, Zhengzhou, China \\
  $^{3}$National University of Defense Technology, Changsha, China \\
  \texttt{\{wangxiaobao,rendah,xiazhiliang,jindi\}@tju.edu.cn} \\
}

\newcommand{\method}{FedGRAG}

\iclrfinalcopy
\begin{document}

\maketitle

\begin{abstract}
Graph Retrieval-Augmented Generation (GraphRAG) leverages centralized knowledge graphs to enhance large language models with structured knowledge retrieval. However, real-world knowledge is inherently distributed across multiple organizations rather than being centrally accessible. Privacy constraints further prohibit sharing entire local knowledge graphs, preventing GraphRAG from effectively retrieving and utilizing distributed knowledge. To this end, we propose Federated GraphRAG (FedGRAG), a unified framework for distributed graph retrieval and knowledge-enhanced generation over multiple clients with knowledge graphs. Specifically, we address two fundamental challenges in federated GraphRAG. First, the absence of a global knowledge graph makes unified retrieval impossible. To address this, we develop a Federated Knowledge Routing mechanism that efficiently identifies relevant clients through lightweight knowledge indices without exposing local graph structures. Second, local graph memory cannot be shared directly because it contains sensitive entities and knowledge information. We therefore propose a controlled-disclosure evidence construction method. It extracts retrieved evidence subgraphs and converts them into an evidence abstract, enabling evidence conversion and reconstruction on the server. Extensive experiments demonstrate that Federated GraphRAG consistently outperforms existing baselines in both retrieval and downstream generation tasks, establishing a new paradigm for deploying GraphRAG in federated environments.
\end{abstract}

\section{Introduction}
Graph Retrieval-Augmented Generation (GraphRAG) \citep{peng2025graph,hu2025grag,procko2024graph} has recently emerged as a promising paradigm for improving Large Language Models (LLMs) \citep{chang2024survey} by incorporating graph-structured knowledge \citep{Liang2023knowledgegraph, liu2025adaptivefeature, liu2025deeptemporal, liang2025fromconcrete} into the retrieval process. Representative GraphRAG systems, such as Microsoft's GraphRAG \citep{han2024retrieval}, HyperGraphRAG \citep{luo2026hypergraphrag}, and LightRAG \citep{guo2024lightrag}, organize knowledge into graph-structured representations to enable graph retrieval, multi-hop reasoning, and knowledge-enhanced generation. 

Despite their success, existing GraphRAG methods fundamentally rely on a centralized graph repository. This assumes that graph-structured knowledge can be globally constructed, indexed, and retrieved from a unified location. However, in many real-world applications, graph knowledge is inherently distributed across multiple organizations and cannot be centralized due to privacy and ownership \citep{liu2025federatednodelevel,Liu2025federatedgraph,han2026federatedgraphlevel,han2025federatedgraphlevel,liu2026causallyawareattribute}. This motivates the development of RAG methods for distributed settings. Recent Federated Retrieval-Augmented Generation (FedRAG) methods, such as \citep{guerraoui2025efficient, addison2024c, zhao2024frag} alleviate this limitation by enabling collaborative retrieval over distributed knowledge repositories \cite{liao2025fedbrba, fu2023fedhgna}. Nevertheless, They are primarily designed for vector-based retrieval and cannot be directly extended to GraphRAG. Unlike vector representations, graph memories contain rich structural semantics, and effective collaborative retrieval requires not only identifying relevant knowledge sources but also extracting the semantic relationships among distributed graph memories. First, without graph-structured knowledge at local clients, conventional RAG relies mainly on isolated textual retrieval, making it difficult to capture explicit entity relations and multi-hop dependencies. As a result, complex queries often require repeated retrieval over fragmented evidence, increasing retrieval cost while limiting the construction of coherent reasoning paths.

Based on the above analysis, we need to design a federated GraphRAG method that supports evidence retrieval augmentation on graphs and collaborative reasoning across distributed clients simultaneously. Inspired by federated graph learning \citep{liu2025federatednodelevel}, which enables collaborative optimization over graphs distributed across clients, an intuitive approach is for the server to provide a guidance signal to select the correct client for queries. However, this faces the following challenges:
First, the absence of a globally accessible graph makes it difficult to identify relevant and complementary knowledge across clients, forcing queries to search distributed clients individually and introducing substantial retrieval latency and computational overhead. Second, locally retrieved subgraphs contain potentially sensitive entities, relations, and structural information, yet restricting their exchange may obscure the dependencies needed for cross-client reasoning. Evidence integration therefore entails a fundamental tension between preserving query-relevant semantics and relational context and limiting the disclosure of local knowledge.
\begin{wrapfigure}[19]{r}{0.45\textwidth}
    \centering
    \includegraphics[width=0.46\textwidth]{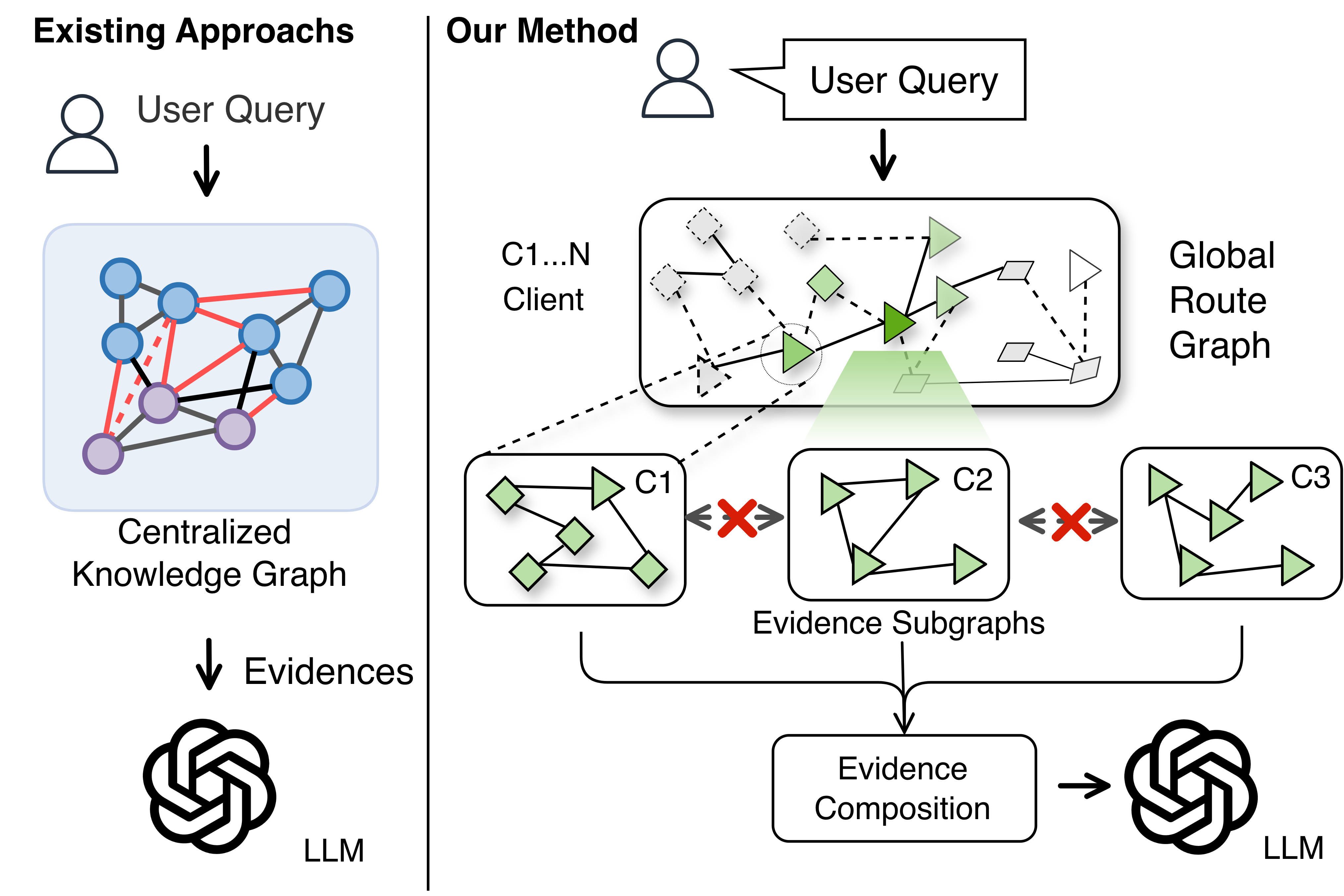}
    \caption{Motivation of Federated GraphRAG. Existing GraphRAG relies on centralized graph memories, while directly extending Federated RAG fails to capture semantic relationships among distributed graph memories.}
    \label{fig:placeholder}
\end{wrapfigure}

To solve these limitations, we propose Federated GraphRAG (\method{}), the first unified framework for distributed graph retrieval and knowledge-enhanced generation over multiple clients. Specifically, we first introduce a Federated Knowledge Routing mechanism that constructs lightweight client-level knowledge indices to efficiently route queries toward the most relevant clients without exposing local entire knowledge graphs. We further propose an evidence abstraction strategy that transforms locally retrieved subgraphs into limited exposure evidence representations, enabling collaborative knowledge utilization while preventing direct leakage of entire graph knowledge. Extensive experiments demonstrate that Federated GraphRAG consistently outperforms existing baselines on both retrieval and downstream generation tasks, establishing a practical paradigm for deploying GraphRAG in federated environments. 
In summary, our contribution can be listed as:


\begin{itemize}
    \item \textbf{New Perspective.} To the best of our knowledge, this is the first work to investigate a federated GraphRAG paradigm, extending GraphRAG from centralized knowledge repositories to distributed environments under privacy and knowledge ownership constraints.
    \item \textbf{Unified Framework.} We propose a unified Federated GraphRAG framework that enables collaborative retrieval-augmented generation over distributed graph-structured knowledge through federated knowledge routing and collaborative evidence abstraction.
    \item \textbf{Superior Performance.} Extensive experiments on multiple benchmarks demonstrate that Federated GraphRAG consistently outperforms representative GraphRAG baselines.
\end{itemize}

\section{Related Work}
\subsection{Graph Retrieval-Augmented Generation}
Retrieval-Augmented Generation (RAG) \citep{gao2023retrieval, han2024retrieval} enhances LLMs by retrieving external knowledge to improve factual accuracy and reasoning capability. Recently, GraphRAG has emerged as a promising extension of RAG by organizing retrieved knowledge into graph-structured memories. Representative methods include Microsoft's GraphRAG~\citep{edge2024local}, which constructs entity-community graphs to facilitate hierarchical retrieval and global reasoning. Similarly, LightRAG~\citep{guo2025lightrag} improves retrieval efficiency through lightweight graph indexing. Moreover, HippoRAG~\citep{gutierrez2024hipporag}  models long-range semantic associations via graph-based memory organization. 
SubgraphRAG~\citep{li2025simple} instead retrieves compact query-relevant subgraphs through lightweight triple scoring and structural-distance modeling, balancing retrieval efficiency with graph reasoning capability. More recently, HippoRAG2~\citep{gutierrez2025rag} extends graph-based memory toward non-parametric continual learning over newly acquired knowledge. HyperGraphRAG~\citep{luo2026hypergraphrag} further generalizes conventional pairwise graph relations to hypergraph structures, representing higher-order $n$-ary relationships through hyperedges for more expressive knowledge retrieval and generation. GFM-RAG~\citep{luo2026gfm} introduces a graph foundation model trained over large-scale heterogeneous knowledge graphs, enabling transferable graph-aware retrieval on unseen datasets without task-specific fine-tuning.


\subsection{Federated Retrieval-Augmented Generation}
Recent studies have extended RAG to federated settings \citep{liu2025federatednodelevel, liu2026personalizedfederated, han2025federatedgraphlevel, han2026federatedgraphlevel}, enabling LLM to leverage distributed knowledge. Federated RAG allows each client to maintain its own knowledge repository and performs collaborative retrieval under the coordination of a central server~\citep{chakraborty2025federated}.
Current FedRAG approaches can be broadly categorized into two directions. Routing-based methods aim to reduce communication overhead by selecting only a subset of relevant clients for each query. Representative examples include RAGRoute \citep{guerraoui2025efficient}, which learns lightweight routing policies for efficient federated search, and FedE4RAG \citep{mao2025privacy}, which systematically investigates the effectiveness of different client selection strategies under distributed retrieval settings. Privacy-oriented methods focus on protecting sensitive knowledge during retrieval. Representative frameworks include C-FedRAG \citep{addison2024c}, and FRAG \citep{zhao2024frag} which improve retrieval security without exposing local knowledge bases. 


\section{Methodology}
In this section, we present Federated GraphRAG (\method{}) in detail, a unified framework for retrieval and knowledge-enhanced generation on distributed knowledge graphs. The core idea is construct semantic cluster routing for different client on the server; routing queries to the clients, which identify and extract evidentiary subgraphs, and the results are then globally integrated to generate the final retrieval output. \method{} consists of three components: (1) federated knowledge routing (2) client evidence representation and (3) collaborative generation. The architecture of \method{} is shown in Fig. \ref{fig:framework}, and its pseudocode appears in \textbf{Appendix}~\ref{app:sec:algorithm}.

\subsection{Notation}
We consider a federated GraphRAG system consisting of a central server and $M$ distributed clients.
Each client owns a private document corpus $\mathcal{D}$, where the local corpora is generally non-IID. Each client utilize $\mathcal{D}$ constructs a knowledge graph $\mathcal{G}=(\mathcal{V},\mathcal{E},\mathbf{X})$, where $\mathcal{V}$, $\mathcal{E}$, and $\mathbf{X}$ denote the node set, edge set, and node attribute matrix, respectively. 


\subsection{Federated Knowledge Routing}
In this section, because knowledge in federated environments is inherently distributed across multiple clients, the server has no prior knowledge of which client contains the information required to answer a given query. Therefore, we propose a knowledge routing mechanism that organizes distributed client knowledge into a unified routing space. The key idea is to represent each client using a set of knowledge prototypes, where each prototype characterizes one representative semantic region covered by the local graph memory. These prototypes are then organized into a global routing space, enabling efficient query routing without exposing local graph memories.

\textbf{Knowledge Prototype Learning}
In this stage, our goal is to capture the representative local knowledge area for domain knowledge prototype generation. 
Specifically, for each client, we first learn the node attribute and graph structure representation, $\mathbf{H}^\psi$ and $\mathbf{H}^\phi$ by a dual encoder $\psi$ and $\phi$ as
\begin{equation}
\mathbf{H}^\psi=\operatorname{\psi}(\mathbf{X}),\quad
\mathbf{H}^\phi=\operatorname{\phi}(\mathbf{A}, \mathbf{X}).
\end{equation}
To jointly exploit semantic attributes and graph topology, we fuse the two representations into a unified node embedding, as
\begin{math}
\mathbf{H}
=
\operatorname{MLP}
(
\mathbf{H}^{\psi}
\Vert
\mathbf{H}^{\phi}
),
\end{math}
where $\Vert$ denotes feature concatenation and MLP projects the fused representation into a unified latent space.
Since local graph memories usually contain multiple semantic regions, we partition the fused node embedding space into $K$ semantic groups using $K$-means clustering,
where each cluster represents a coherent semantic region, collectively providing a compact yet expressive characterization of the client's knowledge space.
The knowledge prototype of each semantic region is then defined as the centroid of the corresponding cluster,
\begin{math}
\mathbf{p}^{k}
=
\frac{1}{|\mathcal{C}_k|}
\sum\nolimits_{\mathbf{h}_j\in\mathcal{C}_k}
\mathbf{h}_j,
\end{math}
where $|\mathcal{C}_k|$ denotes the number of nodes assigned to the $k$-th cluster. Collecting all cluster centroids yields the prototype matrix $\mathbf{P}$, as
\begin{math}
\mathbf{P}
=[\mathbf{p}^{1},\mathbf{p}^{2},\ldots,\mathbf{p}^{K}].
\end{math}
The learned knowledge prototypes explicitly preserve multiple semantic regions contained in each local graph memory. Consequently, heterogeneous client knowledge can be represented in a unified semantic space while maintaining sufficient semantic granularity for query routing. 

\textbf{Global Knowledge Routing Space}
After the local client prototypes are uploaded to the server. The server aims to utilize these prototypes to build a routing graph, which serves as a semantic index of the entire federated system.
In detail, the global routing graph is defined as $ \mathcal{R} = \left( \mathcal{V}^{r}, \mathcal{E}^{r}, \mathbf{X}^{r} \right)$ where $\mathcal{V}^{r}$ and $\mathcal{E}^{r}$ denote the set of nodes and the set of edges, respectively, and $\mathbf{X}^{r}$ denotes the matrix of features of the nodes.
Each routing node corresponds to one knowledge prototype generated by a client, as $ \mathcal{V}^{r} = \{ v_{i}^{k} \mid i=1,\ldots,M,\; k=1,\ldots,K \} $
where $v_i^k$ denotes the $k$-th knowledge prototype of client $i$, whose initial node feature is given by 
$\mathbf{x}_{i}^{k} = \mathbf{p}_{i}^{k}. $
To describe the semantic relationships among distributed knowledge prototypes, we compute the pairwise prototype similarity and connect each prototype with to its top-$L$ most relevant prototypes (KNN algorithm \cite{Bo2020structuraldeep}), as
\begin{equation} 
a_{ij}^{kl} = \operatorname{sim} \left( \mathbf{p}_{i}^{k}, \mathbf{p}_{j}^{l} \right), \quad \mathcal{E}^{r} = \left\{ \left(v_i^k,v_j^l\right) \mid v_j^l \in \operatorname{TopL} \left( a_{ij}^{kl} \right) \right\}. 
\end{equation} 
where $\operatorname{sim}(\cdot)$ denotes cosine similarity. 
The corresponding weighted adjacency matrix is defined as \begin{equation} \mathbf{A}^{r}_{(i,k),(j,l)} = \begin{cases} a_{ij}^{kl}, & \left(v_i^k,v_j^l\right)\in\mathcal{E}^{r}, \\ 0, & \text{otherwise}. \end{cases} \end{equation} 
The route graph $\mathcal{R}$ transforms isolated client prototypes into a semantic structure connected globally across distributed knowledge regions  
without accessing local private graph data. 
This routing space significantly reduces communication cost while establishing a unified semantic foundation for subsequent query routing.

\textbf{Prototype-driven Query Routing.}
For user query $q$, the server first encodes it into a query representation $\mathbf{q}$ using the same encoder as the client. An initial relevance score is computed between the query and each client prototype, as
\begin{math}
\mathbf{r}
=\operatorname{sim}(\mathbf q,\mathbf{P}_i).
\end{math}
We select the $L$ prototypes most relevant to the query as the starting nodes for relevance propagation.
Subsequently, we propagate query relevance on the routing graph over semantically correlated clients. The initial relevance scores are regarded as node features $\mathbf{R}$ on the routing graph, as $\mathbf R=[\mathbf{r}_1,\cdots,\mathbf{r}_M]^\top,$
and graph propagation is performed as
\begin{equation}
\mathbf R^{(t+1)}
=
\alpha \mathbf R^{(0)}(\theta_x)
+
(1-\alpha)\tilde{\mathbf A}^{R}(\theta_e)\mathbf R^{(t)}.
\end{equation}
where $\theta_x$ and $\theta_e$ are learnable parameters.
The propagated relevance scores capture not only the direct similarity between the query and individual clients, but also the semantic dependencies between neighboring knowledge regions. Consequently, clients that are semantically complementary to the query can also receive high routing scores through graph propagation.
After $T$ propagation steps, the routing score of client $i$ is defined by its highest-scoring prototype:
\begin{math}
s_i(q)=\max_{k=\{1,\ldots,K\}}r_{i,k}^{(T)}.
\end{math}
The server selects the $B$ clients with the highest scores:
\begin{math}
\mathcal C_q
=
\operatorname{TopB}_{i\in\{1,\ldots,M\}}s_i(q).
\end{math}

\begin{figure*}
    \centering
    \includegraphics[width=1\linewidth]{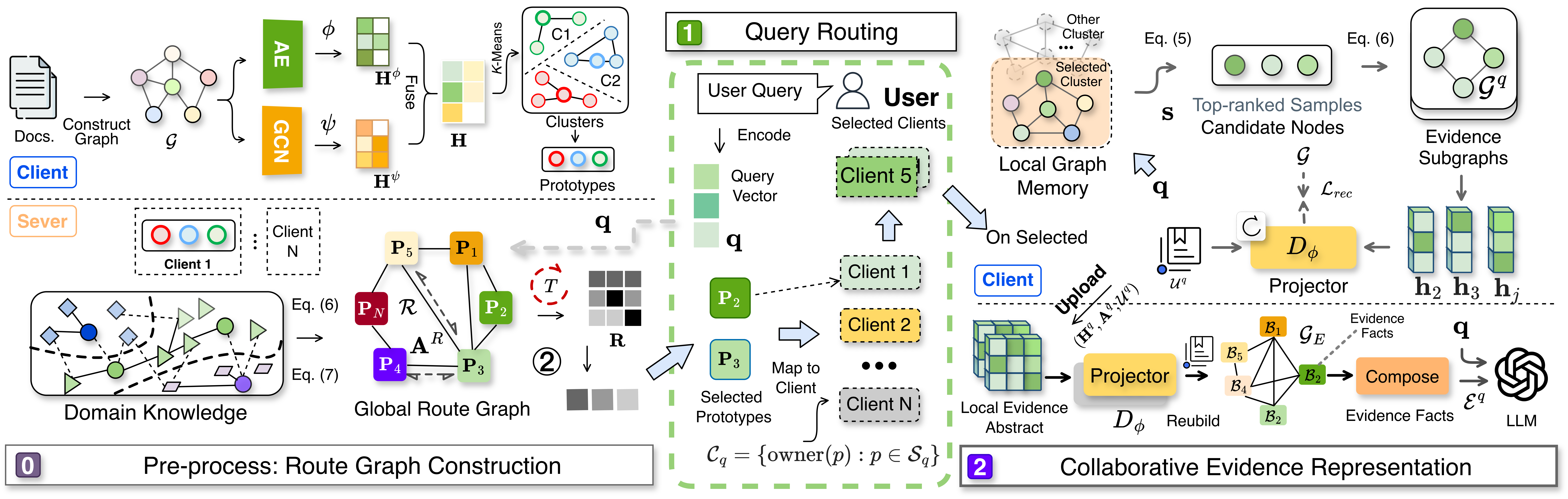}
    \caption{The architecture of \method{}. 
    This process consists of a pre-process and a query phase. During the pre-process phase, the client and server perform collaborative training to extract low-dimensional graph embeddings; the client then uploads prototypes to the server to construct a prototype graph. In the query phase, the server routes the query to the appropriate semantic cluster on the client. Subsequently, an evidence subgraph is constructed within the cluster, and the  server aggregates these reconstructed evidence facts to generate retrieval results.
    }
    \label{fig:framework}
\end{figure*}

\subsection{Client Evidence Representation}
After client routing, the objective of this stage is to construct a unified evidence representation from the retrieved knowledge distributed across multiple selected clients. Each client first extracts a query-relevant evidence subgraph from its local graph memory and map the subgraph into evidence representation that captures both semantic information and structural dependencies. The server then collaboratively integrates them from different clients to form a global evidence affects, which serves as the knowledge context for the subsequent generation module.

\textbf{Local Graph Retrieval}
After receiving the query, each selected client constructs a query-relevant evidence subgraph from its local graph memory. We first compute the semantic relevance between  $\mathbf{q}$ and each local entity  $\mathbf{h}_v$. Then the Top-$N_s$ most relevant entities are selected as seed, as
\begin{equation}
s_i(v \mid q)=\operatorname{sim}(\mathbf{q},\mathbf{h}_v).
\quad
\mathcal{S}^q=\operatorname{TopN_s}_{v\in\mathcal{V}}
s_i(v\mid q).
\end{equation}
Starting from these seed entities, we expand the neighborhood a $H$-hop over its local graph. The resulting candidate node set is defined as
\begin{math}
\mathcal{V}^q=\bigcup\nolimits_{v\in\mathcal{S}^q}\mathcal{N}_{\leq H}(v),
\end{math}
where $\mathcal{N}_{\leq H}(v)$ denotes the set of entities reachable from $v$ within at most $H$ hops. The corresponding evidence subgraph is then constructed as
\begin{equation}
\mathcal{G}^q
=
\left(
\mathcal{V}^q,
\mathcal{E}^q
\right),
\mathcal{E}^q
=
\left\{
(u,r,v)\in\mathcal{E}_i
\mid
u,v\in\mathcal{V}^q
\right\}.
\end{equation}
To prevent irrelevant neighborhood expansion, we retain at most $N_r$ edges with the highest query relevance scores. The resulting subgraph $\mathcal{G}^q$ is kept locally and is subsequently projected into a compact evidence representation for collaboration. 

\paragraph{Local Evidence Abstraction.}
To support cross-client evidence composition while restricting the disclosure scope, we introduce query-conditioned evidence encoding
and reconstruction. For the retrieved evidence subgraph $\mathcal{G}^q$, we first construct a query-specific vocabulary $\mathcal{U}^q$ containing the tokens that appear in its entity and relation attributes. The original
attributes are then converted into discrete token sequences according to $\mathcal{U}^q$ and mapped into continuous attribute embeddings.

We utilize representation $\mathbf{H}^q =\mathbf H[\mathcal V^q,:]$ to define query-conditioned evidence package as $\mathcal{P}^q=(\mathbf{H}^q,\mathbf{A}^q,\mathcal{U}^q).$ The source documents and the remaining local knowledge graph are retained by the client.
On the server, a shared evidence decoder $D_{\phi}$ inverse mapping from the latent representation to the corresponding entity and relation attributes:
\begin{math}
(\hat{\mathbf{X}}^q,\hat{\mathbf{T}}^q)=\operatorname{Decode}_{\mathcal{U}^q}
[D_{\phi}(\mathbf{H}^q,\mathbf{A}^q)],
\end{math}
where $\hat{\mathbf{X}}^q$ and
$\hat{\mathbf{R}}^q$ denote the reconstructed entity and relation attributes, respectively. Combining the reconstructed attributes with the disclosed adjacency structure yields the reconstructed evidence subgraph $\hat{\mathcal{G}}^q$
The reconstructed subgraph is later converted into a set of facts from the evidence $\widehat{\mathcal{B}}^q$ for the composition of the evidence between clients. 
The encoder and decoder are trained by minimizing the reconstruction loss over the vocabulary-encoded entity and relation attributes:
\begin{equation}
\mathcal{L}_{\mathrm{rec}}
=
\operatorname{CE}(
\hat{\mathbf{X}}^q,
\mathbf{X}^q)
+
\operatorname{CE}(
\hat{\mathbf{T}}^q,
\mathbf{T}^q),
\end{equation}
where $\operatorname{CE}(\cdot)$ denotes the token-level reconstruct loss.

\textbf{Collaborative Evidence Composition.}
The server reconstructs evidence facts $\hat{\mathcal B}$ using the shared decoder and aggregates them as
\begin{math}
\hat{\mathcal B}^{q}
=
\bigcup_{i\in\mathcal C_q}\hat{\mathcal B}_i^{q}.
\end{math}
To connect evidence across clients, we construct an evidence
interaction graph
$\mathcal G_E^q=(\hat{\mathcal B}^q,\mathcal A_E^q)$,
where nodes represent reconstructed facts and edges capture potential reasoning dependencies. 
Starting from facts with high semantic similarity to the query, the server traverses graph edges to collect connected supporting facts across clients.
The resulting evidence paths are serialized into a textual
context $\mathcal E^q$ in traversal order, with repeated facts removed and client-level provenance retained. Finally, an external LLM generates the answer from the query and composed evidence:
\begin{equation}
\hat y
=
\operatorname{LLM}
\left(
\operatorname{Prompt}(q,\mathcal E^q)
\right).
\end{equation}
\subsection{Joint Optimization Objective}
We optimize the routing module and local evidence encoders using two complementary objectives. 

\textbf{Routing Loss (Server).}
For a query $q$, the routing module produces a relevance score $R^q$ for each of the $M \times K$ clusters.
We adopt a pairwise ranking loss to encourage positive clusters to receive higher scores than negative clusters:
\begin{equation}
\mathcal{L}_{\mathrm{route}}
=
\sum\nolimits_{i\in\mathcal{C}_q^{+}}
\sum\nolimits_{j\in\mathcal{C}_q^{-}}
\max(0,\gamma-s_i(q)+s_j(q)),
\end{equation}
where $\gamma$ is a  hyperparameter. $\mathcal{C}_q^{+}$ and $\mathcal{C}_q^{-}$ denote the positive and negative cluster sets of clients.

\textbf{Global Evidence-Answer Alignment.}
We design a global evidence-answer alignment objective. For each training pair $(q,y)$, the evidence interaction graph decoder $D_\phi$ produces the global reconstructed node attribute $\hat{\mathbf{X}}_g^q$. We obtain a fixed semantic target from the reference answer:
\begin{equation}
\mathbf{a}_y = \operatorname{Enc}_{\mathrm{fixed}}(y),
\qquad
\hat{\mathbf{X}}_g^q = D_\phi(\mathbf{H}_g^q),
\end{equation}
where $\operatorname{Enc}_{\mathrm{fixed}}$ is a frozen text encoder.
We encourage the global evidence representation to align with the correct answer while distinguishing it from negative answers:
\begin{equation}
\mathcal{L}_{\mathrm{g}}(q,y)
=
-\log
\frac{
\exp(
\operatorname{sim}(\hat{\mathbf{x}}_g^q,\mathbf{a}_y)/\gamma)
}{
\exp(
\operatorname{sim}(\hat{\mathbf{x}}_g^q,\mathbf{a}_y)/\gamma)
+
\sum_{y^{-}\in\mathcal{Y}_q^{-}}
\exp(
\operatorname{sim}(
\hat{\mathbf{x}}_g^q,\mathbf{a}_{y^{-}})/\gamma)
},
\end{equation}
where $\mathcal{Y}_q^{-}$ contains negative answers sampled from other training examples.
Known equivalent answers and other valid answers to $q$ are excluded from the negative set.
This objective optimizes $D_\phi$, encouraging the composed evidence to retain information predictive of the target answer.



\textbf{Local Evidence Alignment.}
Each selected client retrieves local evidence subgraphs $\mathcal{G}^q$ and their representation $\mathbf{E}^q$.
To preserve query-relevant semantics, we define $\mathcal{P}_q\subseteq\mathcal{C}_q\cap\mathcal{C}_q^{+}$, where $\mathcal{C}_q^{+}$
contains clients whose retrieved subgraphs provide positive evidence. we align positive evidence with $\mathbf{q}$.
The local alignment loss is as
\begin{equation}
\mathcal{L}_{\mathrm{l}}
=
-\sum\nolimits_{i\in\mathcal{P}_q}
\log
\frac{
\exp(
\operatorname{sim}(\mathbf{h}_i^q,\mathbf{q})/\gamma)
}{
\exp(
\operatorname{sim}(\mathbf{h}_i^q,\mathbf{q})/\gamma)
+
\sum_{\mathbf{h}^{-}\in\mathcal{N}_q}
\exp(
\operatorname{sim}(\mathbf{h}^{-},\mathbf{q})/\gamma)
},
\end{equation}
where $\mathcal{N}_q$ contains evidence representations treated as irrelevant to $q$. Negative evidence may be sampled from irrelevant clients, with known positive evidence for $q$ excluded.

\textbf{Overall Objective.}
The server and clients are trained with the following objective:
\begin{equation}
\mathcal{L}_{s} = \mathcal{L}_g + \mathcal{L}_{route} \quad
\mathcal{L}_{c}
=
\frac{1}{|\mathcal{D}|}
\sum\nolimits_{(q,y)\in\mathcal{D}}
\lambda_l\mathcal{L}_{\mathrm{l}} + \mathcal{L}_{rec}.
\end{equation}
The routing loss improves client relevance estimation, while the local alignment loss encourages local evidence representations to preserve query-relevant semantics.

\begin{table*}[tb!]
\centering
\renewcommand{\arraystretch}{0.73}
\resizebox{\textwidth}{!}{
\begin{tabular}{l|cccc|cccc|cccc}
\toprule
\multirow{2}{*}{\textbf{Models}} & 
\multicolumn{4}{c}{\textbf{HotpotQA}} & 
\multicolumn{4}{|c}{\textbf{MuSiQue}} & 
\multicolumn{4}{|c}{\textbf{2Wiki
}} \\\cmidrule{2-13}
&\multicolumn{12}{|c}{5 Clients} \\
\cmidrule{2-13}
& R@2 & R@5 & EM & F1 
& R@2 & R@5 & EM & F1 
& R@2 & R@5 & EM & F1  \\
\midrule
GraphRAG
&10.9 &13.4 & 39.1 & 51.2 
&8.9&14.7&13.1&21.9
&17.1&20.6&31.0&34.7
\\
LightRAG& 12.3 & 17.7& 37.8&40.5&
11.9& 12.2&18.5&24.1&13.2&19.0& 22.3 &23.9\\
HippoRAG &9.6&21.0&28.7&31.1&12.3&13.1&20.2&22.6&
12.9&22.4&33.9&37.3
\\
SubgraphRAG&
15.0& 19.1 & 26.7 & 30.3
&7.0 & 16.2& 23.1 & 25.7& 14.4 & 18.6 &25.0 &38.2

\\
HippoRAG2 
&15.3 & 16.1& 21.3&28.7
& 10.4 & 15.6 &18.1&23.3
& 15.8 & 16.0 &34.6& 35.2
\\
HyperGRAG& 18.9 &29.8&43.0 
& 50.1&9.6&16.4 & 18.5&21.3&
12.6 & 15.1 & 31.4 & 34.3
\\
GFM-RAG & 22.4 & 38.9 & 57.2 & 69.1
        & 16.1 & 25.3 & 24.0 & 31.2
        & 18.4 & 29.8 & 43.6 & 45.9 \\
\midrule
FRAG & 21.8 & 32.6& 55.5 & 59.1 & 14.5& 20.0 & 22.3&25.6
&11.9&14.8&28.3&31.6\\
FedE4RAG& 13.5 & 19.7& 24.6 &38.8
&10.0&17.3&24.7&29.6
&14.6&18.4&32.1&36.9
\\
C-FedRAG&17.0&24.8&38.2&47.9
&12.7&15.3&21.8&23.4
&15.3&17.2&22.1&28.4\\
FedMosaic& 22.0 & 36.9 & 57.8 & 66.5 & 15.7 & 23.4 & 24.1 & 30.6 & 17.8 & 28.3 & 42.6 & 46.5 

\\\midrule
\textbf{Ours}&25.1  &44.7&64.0&76.4&18.2&28.7&26.6&34.2&20.5&34.6&48.3&49.2\\
\midrule
\multicolumn{13}{c}{10 Clients} \\
\midrule
GraphRAG
&13.1&23.0&30.0&41.5
&12.5&15.4&18.3&26.2
&5.9&13.6&38.8&41.6
\\
LightRAG&9.1&11.2&25.9&30.7&11.7&12.5&13.1&25.9& 7.4&11.2&29.1&33.8
\\
HippoRAG &13.8&19.5&21.5&28.7
&9.6&13.5&17.4&18.0 
&9.0&13.5&31.1&36.7\\
SubgraphRAG&10.5 & 13.3 & 18.9 & 21.0
&8.8&13.2&22.7&24.9
&9.7&12.4&27.6&28.5
\\
HippoRAG2&12.7&21.9& 26.4&27.7&
10.9&12.6&21.7&24.5&8.5 &11.2 &33.1 &38.0\\
HyperGRAG
&11.8&15.0&33.7&38.1
&9.7&15.2&18.3&22.9
&10.4&16.7&34.4&42.5\\
GFM-RAG & 16.8 & 27.4 & 41.2 & 46.8
        & 12.9 & 20.1 & 22.8 & 29.4
        & 10.1 & 15.7 & 44.6 & 46.3 \\
        \midrule
FRAG&16.4&19.2&29.8&34.0 &7.5&10.4&17.6&19.3
&7.6&13.8&26.1&30.5
\\
FedE4RAG
&15.9&23.0&27.5&33.4
&9.7&9.9&21.2&24.2
&8.5&9.4&28.6&29.0\\
C-FedRAG&
6.3&13.2&24.4&31.5&
6.5&14.8&17.9&22.3&
9.2&16.9&37.4&41.3\\
FedMosaic & 17.1 & 27.3 & 40.9 & 47.6 & 12.8 & 20.7 & 23.6 & 29.8 & 10.4 & 16.1 & 43.8 & 48.1 
\\\midrule
\textbf{Ours}       &19.4&31.7&46.8&52.5&14.7&23.9&25.5&33.0&11.8&18.4&50.3&51.2\\
\midrule
\multicolumn{13}{c}{20 Clients} \\
\midrule
GraphRAG&11.8&19.4&28.5&41.8&3.5&6.0&27.0&28.1&5.1&8.2&11.6&24.3\\
LightRAG&8.7&9.1&27.0&30.6 & 8.8 & 14.3 & 22.1 & 27.8 & 3.1 &4.5 &10.7&17.5\\
HippoRAG&7.9&11.2&22.8&29.0
&\multicolumn{4}{c|}{OOM}
&\multicolumn{4}{c}{OOM}
\\
SubgraphRAG&12.3&16.9 &31.8 & 36.7&\multicolumn{4}{c|}{OOM}
&\multicolumn{4}{c}{OOM}\\
HippoRAG2& 11.8& 14.0&19.6&23.7&7.1&12.2&28.3&34.6
&7.2&8.5&13.1&19.0\\
HyperGRAG& 13.5 & 18.9 & 23.5 & 31.0 & 6.9&17.2&32.3&35.8
&\multicolumn{4}{c|}{OOM}
\\
GFM-RAG & 14.9 & 23.8 & 39.5 & 47.6
        & 10.7 & 15.4 & 36.2 & 36.9
        & 7.1 & 10.6 & 15.8 & 46.4 \\
\midrule
FRAG&13.5&18.8&32.6 &34.9&8.7 &12.5 &33.4&34.8&8.0&11.0&15.3&26.7\\
FedE4RAG
&10.6&18.7&24.4&37.5
&9.5&13.8&31.2&33.1
&6.7&9.4&14.0&17.9
\\
C-FedRAG
&13.6&17.2&23.1&40.6
& 9.4&15.2&29.9&31.0
&7.7&11.3&12.4&15.8\\
FedMosaic& 15.3 & 23.6 & 38.7 & 47.2 & 10.8 & 16.2 & 35.7 & 38.8 & 7.4 & 11.6 & 16.4 & 24.2 
\\\midrule
\textbf{Ours} &17.4&27.5
&44.2&53.9&12.4&17.9& 40.6&40.6
&8.3&12.4&17.9&51.1\\
\bottomrule
\end{tabular}}
\caption{Performance comparison with different state-of-the-art GraphRAG methods on three datasets.}
\label{tab:compare:main}
\end{table*}

\subsection{Efficiency Analysis}
The global routing graph contains $P=M\times K$ prototype nodes. Graph construction requires $\mathcal{O}(P^2 d)$ operations by computing all pairwise prototype similarities. After connecting each prototype to its top-$L$ neighbors, the routing graph contains approximately $\mathcal{O}(PL)$ edges. For each query, initial query-prototype matching requires $\mathcal{O}(Pd)$ operations, while $T$ rounds of graph-based relevance propagation require approximately $\mathcal{O}(TPL)$ operations.
After selecting $B$ clients, constructing the evidence interaction graph requires $\mathcal{O}(B^2d)$ operations. Accordingly, the online computational complexity can be summarized as $\mathcal{O}(MKd+TMKL+B^2d)$.
\section{Experiments}

\subsection{Experimental setup}
\noindent\textbf{Benchmark Datasets.}
To comprehensively evaluate \method{}, we conduct experiments on multiple widely used GraphRAG benchmarks, including HotpotQA, MuSiQue, and 2Wiki, aligning with GFM-RAG. Since they are designed for centralized settings, we construct federated scenarios by partitioning the original graph knowledge into multiple clients while preserving realistic knowledge heterogeneity to simulate practical distributed environments. (Please see \textbf{Appendix} \ref{app:sec:datasets} for details.)

\noindent\textbf{Baseline Methods.}
We compare FedGRAG with representative baselines from three categories. (1) centralized GraphRAG methods, including GraphRAG (MS) \citep{edge2024local}, LightRAG \citep{guo2025lightrag}, HippoRAG \citep{gutierrez2024hipporag}, SubgraphRAG \citep{li2025simple},
HyperGraphRAG \citep{luo2026hypergraphrag}
HippoRAG2 \citep{gutierrez2025rag}, GFM-RAG \cite{luo2026gfm}. (2) Federated RAG methods, including FRAG \cite{zhao2024frag}, FedE4RAG \cite{mao2025privacy}, C-FedRAG \cite{addison2024c}, FedMosaic \cite{liang2026fedmosaic} which perform retrieval over distributed knowledge sources.  (Please see \textbf{Appendix} \ref{app:sec:baselines} for details.)

\noindent\textbf{Evaluation Metrics.}
We evaluate the FedGRAG from three complementary perspectives. For retrieval quality, we adopt Recall@2 (R@2) and Recall@5 (R@5) to measure the effectiveness of retrieving relevant graph knowledge. For downstream generation, we report Exact Match (EM) and F1 scores to evaluate the correctness and semantic overlap between the generated answers and the ground-truth responses. (Please see \textbf{Appendix} \ref{app:sec:evaluation} for details.)

\noindent\textbf{Implementation Details.}
All experiments are implemented using PyTorch 2.4.0 and conducted on a server equipped with NVIDIA 4090D GPUs. All methods share the same LLM, graph retrieval pipeline, and retrieval configuration. The federated setting adopts the same client partition strategy across all methods, and hyperparameters are selected based on the validation set to ensure a fair and reproducible evaluation. (Please see \textbf{Appendix} \ref{app:sec:imple} for details.)

\subsection{Comparison Experiments}
Table~\ref{tab:compare:main} reports the overall performance comparison. 
1) Compared with federated GraphRAG methods, \method{} consistently achieves superior performance across all datasets. Although these methods are adapted to the federated setting, they perform retrieval independently within selected clients and lack explicit modeling of semantic dependencies among distributed graph memories. In contrast, FedGRAG constructs a global semantic routing graph over client prototypes, enabling more accurate client selection and more comprehensive cross-client knowledge discovery.
2) Compared with federated RAG methods, \method{} further improves both retrieval and generation performance. Existing federated RAG methods retrieve textual evidence from distributed clients without exploiting graph structures or semantic relationships among clients. By leveraging graph-structured knowledge and collaborative evidence composition, FedGRAG effectively captures complementary evidence from multiple clients, leading to more informative context for answer generation.
\begin{wrapfigure}[11]{r}
{0.5\textwidth}
\vspace{-10pt}
    \centering
\includegraphics[width=1\linewidth]{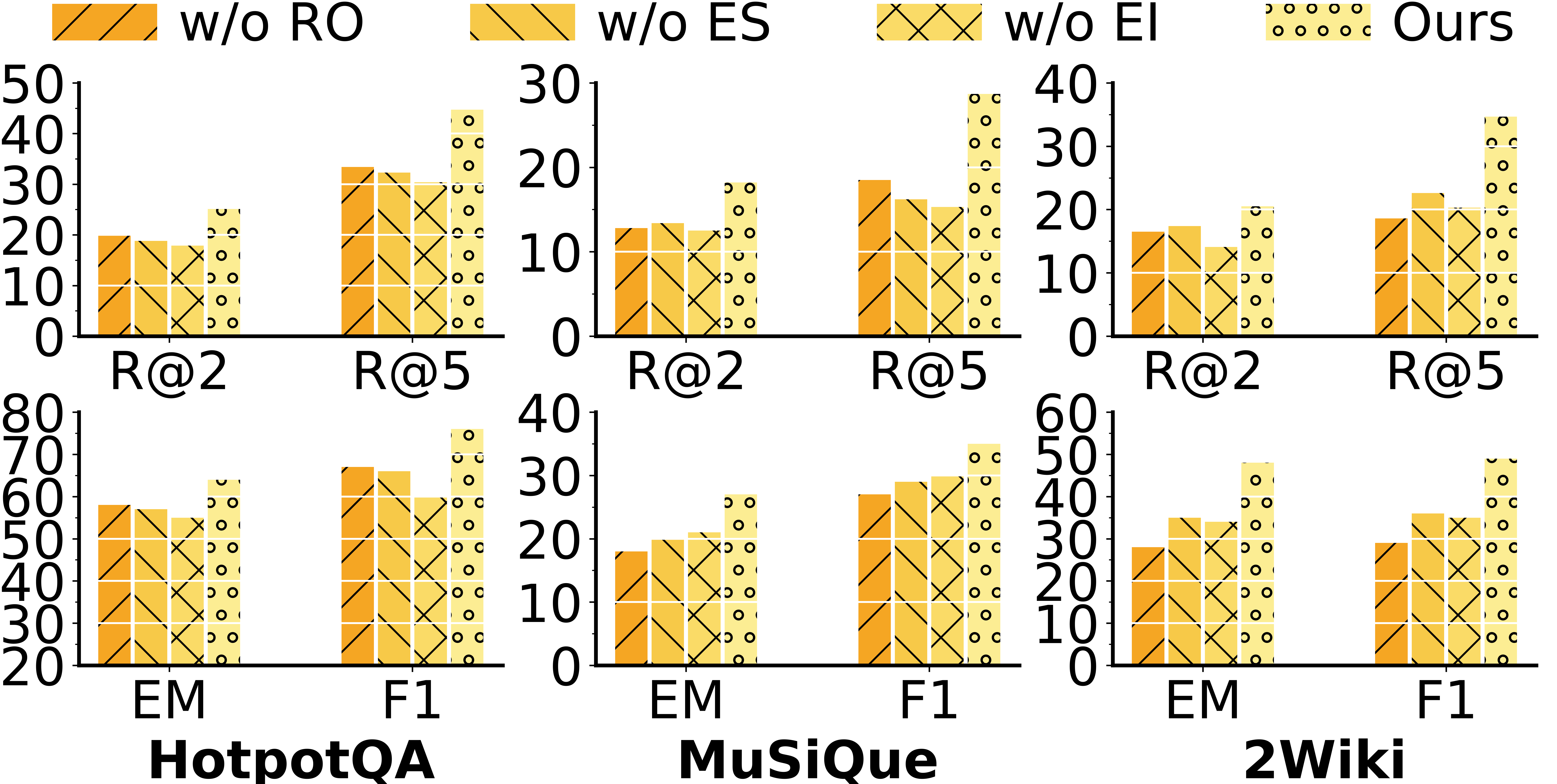}
    \caption{Model performance of \method{} and its three variants.}
    \label{fig:ablation}
\end{wrapfigure}
3) Compared with the GFM-based GraphRAG method, FedGRAG also achieves consistent improvements. Although GFM-based RAG enhances graph representation through graph foundation models, it mainly focuses on improving local graph representations and does not explicitly model semantic collaboration across distributed clients. Our semantic routing graph and collaborative evidence composition enable effective cross-client knowledge integration, resulting in higher retrieval accuracy and better downstream generation performance.

\subsection{Ablation Study}
\begin{wraptable}[11]{r}
{0.60\linewidth}
\vspace{-25pt}
\caption{Client routing performance on three multi-hop QA datasets (\%) under 10 clients.}
\label{tab:route}
\renewcommand{\arraystretch}{0.73}
\resizebox{0.58\columnwidth}{!}{
\begin{tabular}{llcccc}
\toprule
Dataset & Method
& R@2 
& Hit@2 
& R@5 
& Hit@5  \\
\midrule
\multirow{3}{*}{HotpotQA}
& Random Routing  & 7.2 & 31.6 & 9.4 & 34.0 \\
& Direct Matching & 8.5 & 42.4 & 12.3 & 46.2 \\
& Ours & 19.4&66.7& 31.7 & 71.4 \\
\midrule
\multirow{3}{*}{MuSiQue}
& Random Routing  & 6.9 & 13.2 &12.0 & 34.3 \\
& Direct Matching & 8.8 & 25.1 & 14.6 & 42.2 \\
& Ours & 14.7 & 68.6  & 23.9 & 73.2 \\
\midrule
\multirow{3}{*}{2Wiki}
& Random Routing  & 9.5 & 33.4 & 11.4& 36.8 \\
& Direct Matching & 7.7& 26.9 & 12.6& 34.7 \\
& Ours & 11.8 & 45.9
& 18.4 & 52.3 \\
\bottomrule
\end{tabular}}
\end{wraptable}
\textbf{Module Ablation.}
To examine the contribution of the proposed components, we
evaluate three variants of \method{}:
(1) \textbf{w/o RO}, which remove route module;
(2) \textbf{w/o ES}, which replaces evidence subgraph retrieval
with direct similarity-based text retrieval; and
(3) \textbf{w/o EI}, which removes the global evidence interaction
graph and directly concatenates the retrieved evidence as
input to the LLM.
As shown in Fig.~\ref{fig:ablation}, the complete \method{} consistently achieves the best performance, while Removing any of the components will result in a decline in model performance. These results support the complementary roles of the proposed components in distributed knowledge retrieval and reasoning.
The performance degradation of \textbf{w/o ES} suggests that providing more retrieved text does not necessarily yield more
useful evidence: similarity-based retrieval can introduce irrelevant content and overlook structural connections needed for reasoning. 
Moreover, directly providing raw text increases source-content
exposure compared with releasing query-relevant reconstructed facts. Removing the evidence interaction graph (\textbf{w/o EI}) also reduces performance.
Direct concatenation leaves cross-client evidence connections implicit, whereas graph-guided composition organizes related
facts into connected evidence chains, facilitating multi-hop
reasoning across clients.

\textbf{Routing Strategy Ablation}
To evaluate whether the federated routing effectively identifies relevant knowledge sources, We compare it with the following strategies, random routing: randomly selects a fixed number of clients for each query, direct routing: represents each client using the mean embedding of its local graph nodes and selects clients according to query-prototype similarity.
As shown in Table~\ref{tab:route}, \method{} consistently outperforms both baselines. 
The improvement over Random Routing demonstrates the importance of query-aware client selection, while the gain over Direct Routing indicates that modeling inter-client relationships provides additional benefits beyond independent semantic matching, particularly when relevant evidence is distributed across multiple clients.

\subsection{Hyperparameter Sensitivity Analysis}

\begin{wrapfigure}[10]{r}{0.6\textwidth}
\vspace{-10pt}
    \centering   
    
    \includegraphics[width=1\linewidth]{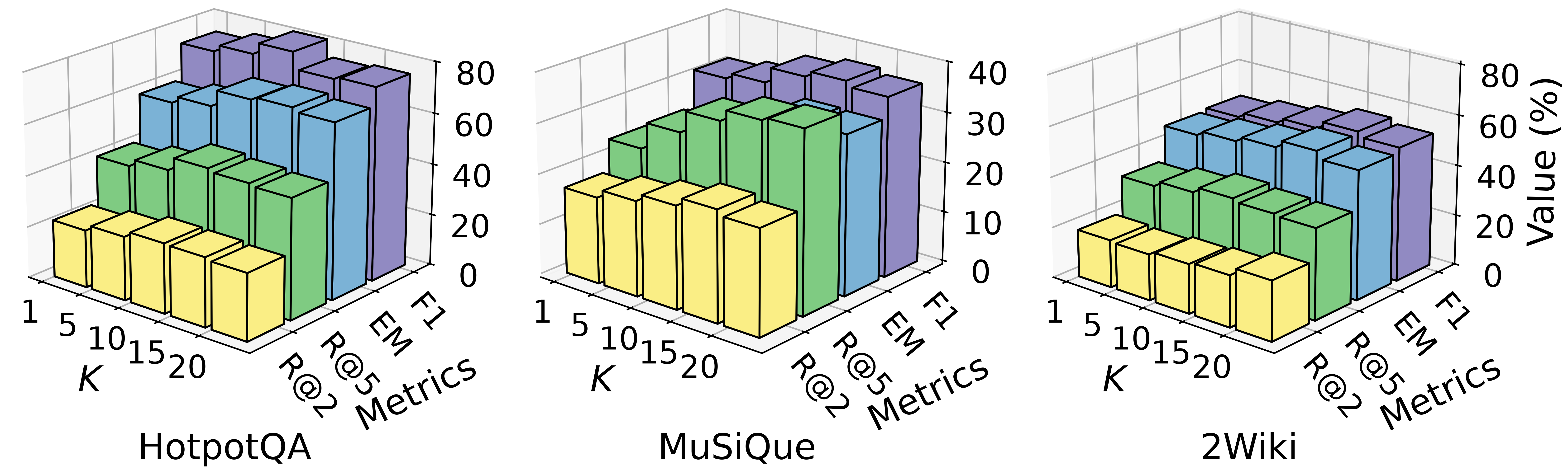}
    \caption{Model performance under varying $K \in [1, 20]$.}
    \label{fig:hyper-k}
\end{wrapfigure}
\textbf{Effect of Prototype Number $K$.} We vary $K$ from 1 to 20 to assess its influence on model performance. As shown in Fig. \ref{fig:hyper-k}, increasing $K$ initially improves performance because more semantic regions are preserved, leading to finer-grained client routing. However, excessive prototypes introduce redundant semantic partitions and increase routing ambiguity, resulting in marginal performance degradation. Therefore, we choose different settings for each datasets as a trade-off between routing accuracy and communication efficiency.

\begin{wrapfigure}[9]{r}{0.5\linewidth}
    \centering
    \vspace{-10pt}
    \includegraphics[width=1\linewidth]{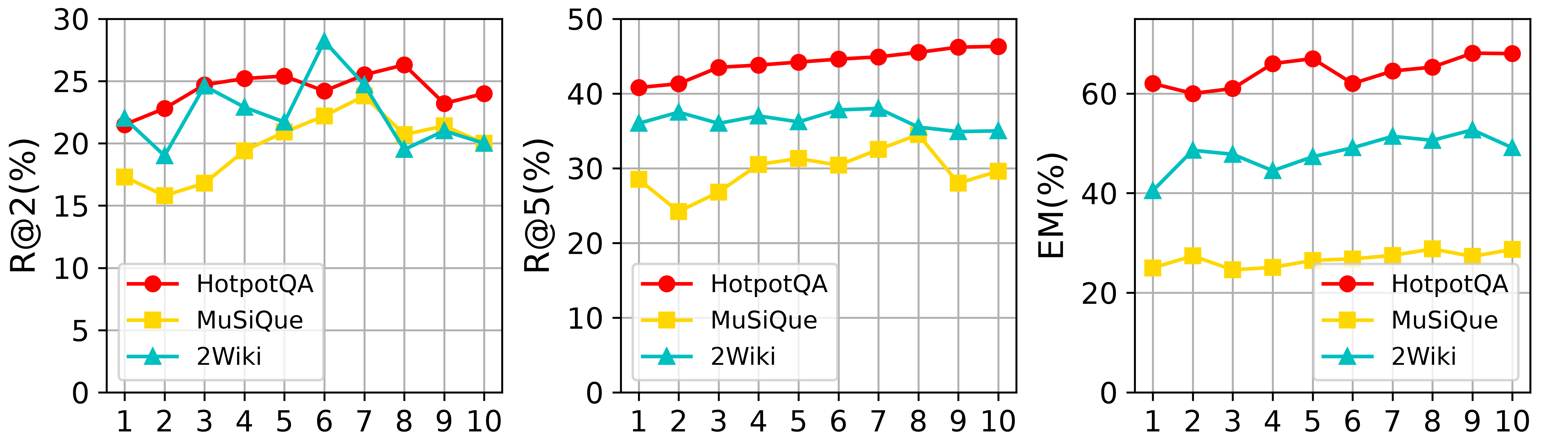}
    \caption{Model performance on varying $L$ $\in$ [$1$, $10$] under three datasets.}
    \label{fig:hyper-l}
\end{wrapfigure}

\noindent\textbf{Effect of Routing Graph Connectivity.}
We further investigate the neighborhood size $L$ used for constructing the routing graph. A small $L$ results in sparse prototype connections, limiting semantic propagation across clients. Conversely, an excessively large L introduces noisy semantic relations that dilute query relevance. As illustrated in Fig. \ref{fig:hyper-l}, the best performance is achieved at near $L$=6, demonstrating that moderate graph connectivity provides the most effective balance between semantic exploration and routing precision.

\begin{wraptable}[11]{r}{0.42\textwidth}
\centering
\renewcommand{\arraystretch}{0.73}
\footnotesize
\vspace{-15pt}
\setlength{\tabcolsep}{2pt}
\caption{Communication overhead comparison under the HotpotQA datasets. ``T'' and ``C'' denote Time (s) and Cost (KB), respectively. }
\label{tab:overhead}
\resizebox{0.4\columnwidth}{!}{
\begin{tabular}{lcccccc}
\toprule[1pt]
\multirow{2}{*}{Clients} & 
\multicolumn{2}{c}{Ours} & 
\multicolumn{2}{c}{FRAG} & 
\multicolumn{2}{c}{C-FedRAG} \\
\cmidrule(lr){2-3}\cmidrule(lr){4-5}\cmidrule(lr){6-7}
 &  T (s) & C (KB) & T (s) & C (KB) & T (s) & C (KB)\\
\midrule
 1 &24.8&50.3& 130.3 & 213.0 & 37.1 & 66.4 \\ 
 2 &51.3&100.6&270.2 & 256.4 & 68.5  & 134.9 \\
 3 &78.6&150.9&300.4 & 375.5 & 105.3 & 201.5 \\
4 &110.1&201.1&418.3& 420.3 & 137.0  & 274.6 \\
 5 &140.9&256.4&566.5 & 458.5 & 174.5  & 355.3 \\
\bottomrule[1pt]
\end{tabular}}
\end{wraptable}
\subsection{Communication Cost and Runtime Efficiency}
To evaluate the practical deployment cost of federated graph retrieval, we measure communication overhead and end-to-end inference latency on HotpotQA under varying numbers of clients. As shown in Table~\ref{tab:overhead}, our framework consistently achieves lower communication cost and inference latency than both FRAG and C-FedRAG across all evaluated settings. As the number of clients increases, its communication overhead grows approximately linearly, while its runtime advantage over the baselines is maintained. These results demonstrate the efficiency of our framework within the evaluated client range and support its practical applicability to federated retrieval scenarios with limited communication resources and latency requirements.
\section{Conclusion}
This paper presented \method{}, a unified framework for distributed GraphRAG over distributed knowledge graphs. By jointly introducing federated knowledge routing and evidence abstraction, the proposed framework effectively enables collaborative knowledge retrieval and utilization without exposing local knowledge graphs. Compared with existing centralized GraphRAG methods, \method{} extends graph retrieval to federated environments while reducing data privacy risks. Extensive experiments demonstrate its superiority in both retrieval quality and generation across diverse federated settings. We believe \method{} provides a practical and scalable paradigm for knowledge-enhanced generation over distributed knowledge ecosystems and offers new opportunities for deploying GraphRAG in privacy-sensitive real-world applications.

\subsection*{AI use statement}
Generative AI tools were used to assist with language editing, improving the readability of the manuscript, LaTeX formatting, and debugging author-written
code. They were not used to develop the proposed methodology, design the
experiments, formulate or prove mathematical claims, process the datasets, or interpret the experimental results. All AI-assisted content was carefully reviewed and verified by the authors. The authors take full responsibility for
the final content of this work.

\bibliography{ref}
\bibliographystyle{iclr2027_conference}

\appendix
\section{Detailed process of the algorithm}
\label{app:sec:algorithm}
The preprocessing and retrieval procedures are summarized in Algorithms~\ref{alg:preprocess} and~\ref{alg:retrieval}, respectively.
\section{Details of Experimental Settings}
\subsection{Detail of Datasets}
\label{app:sec:datasets}
We consider three question answering benchmarks: HotpotQA, MuSiQue, and 2Wiki. These benchmarks require integrating evidence across multiple documents, making them suitable for evaluating knowledge retrieval and evidence aggregation in FedGRAG.To adapt these centralized benchmarks to the federated setting, we distribute the graph knowledge across multiple clients, each maintaining its local knowledge graph. This setting allows us to evaluate retrieval and reasoning over distributed knowledge sources. Table~\ref{tab:dataset_statistics} summarizes the statistics of the training data and the associated knowledge graphs.

\subsection{Details of Baselines}
\label{app:sec:baselines}

We consider representative methods from two categories:
graph-based retrieval-augmented generation and federated
retrieval-augmented generation.

\paragraph{Graph-based RAG Methods.}
\begin{itemize}
    \item \textbf{GraphRAG (MS)} constructs an entity graph
    from source documents and generates hierarchical community
    summaries to support query-focused knowledge aggregation.

    \item \textbf{LightRAG} combines graph-based indexing with
    dual-level retrieval to capture both specific entities and
    broader relational information.

    \item \textbf{HippoRAG} integrates knowledge graphs with
    Personalized PageRank to retrieve related evidence through
    associative connections across documents.

    \item \textbf{SubgraphRAG} retrieves query-relevant subgraphs
    using a lightweight triple-scoring model enhanced with
    directional structural distance features.

    \item \textbf{HippoRAG2} extends HippoRAG through deeper
    passage integration and LLM-assisted retrieval, improving
    the interaction between passage-level information and
    graph-based associative retrieval.

    \item \textbf{HyperGraphRAG} represents knowledge using
    hypergraphs, allowing a single relation to connect multiple
    entities and preserving higher-order relational information.

    \item \textbf{GFM-RAG} employs a graph foundation model
    trained across multiple knowledge graphs to learn
    transferable query-conditioned graph retrieval.
\end{itemize}

\paragraph{Federated RAG Methods.}
\begin{itemize}
    \item \textbf{FRAG} supports collaborative approximate
    nearest-neighbor search over distributed vector databases,
    using homomorphic encryption to protect query vectors
    and stored representations.

    \item \textbf{FedE4RAG} collaboratively trains local
    retrieval models through federated embedding learning,
    incorporating knowledge distillation and homomorphic
    encryption for model collaboration.

    \item \textbf{C-FedRAG} coordinates retrieval across
    distributed data providers and uses confidential computing
    to protect the retrieved context during processing.

    \item \textbf{FedMosaic} encodes local documents into
    parametric adapters and selectively aggregates relevant
    adapters to support generation without transmitting
    raw documents.
\end{itemize}

\paragraph{Adaptation to the Federated Setting.}
For graph-based baselines, indexing and retrieval are performed within individual clients using their assigned local knowledge.
The federated variants use random route strategy to select clients and average strategy to combine local retrieval outputs. We preserve each method's core indexing and retrieval mechanism while using the same corpus partitions and evaluation queries. These configurations are distinguished from centralized execution over the complete corpus. Any changes to the original training, generation,
or security mechanisms are explicitly documented.

\subsection{Details of Evaluation Metrics}
\label{app:sec:evaluation}
We evaluate FedGRAG from three perspectives: evidence retrieval, client routing, and downstream answer generation. All metrics are computed for each query and then averaged over the complete test set. Higher values indicate better performance.

\paragraph{Recall@$k$.}
Recall@$k$ measures the proportion of ground-truth supporting evidence covered by the top-$k$ retrieved results. We report Recall@2 and Recall@5 to evaluate retrieval quality under different evidence budgets. A higher Recall@$k$ indicates that the retrieved results contain more evidence required to answer the query.

\paragraph{Hit@$k$.}
Hit@$k$ evaluates the effectiveness of client routing. A query is considered successfully routed when at least one of the top-$k$ selected clients contains ground-truth supporting evidence. We report Hit@2 and Hit@5 and average the results across all test queries.

\paragraph{Exact Match.}
Exact Match (EM) measures the percentage of generated answers that exactly match a ground-truth answer after standard text normalization. Following common question-answering evaluation protocols, the normalization procedure converts text to lowercase and removes punctuation, articles, and redundant whitespace. When multiple reference answers are available, a prediction is considered correct if it matches any valid reference.

\paragraph{Token-level F1.}
Token-level F1 measures the token overlap between the generated answer and the ground-truth answer by combining token-level precision and recall. Compared with EM, F1 assigns partial credit to answers that overlap substantially with the reference but do not match it exactly. When multiple reference answers are provided, we report the maximum F1 score obtained against all valid references.

\subsection{Detail of Implementation}
\label{app:sec:imple}

\paragraph{Implementation Environment.}
All experiments are implemented in Python using PyTorch~2.4.0 and conducted on a server equipped with NVIDIA RTX 4090D GPUs. The graph-related operations, including neighborhood expansion, prototype-graph construction, and evidence-graph propagation, are implemented using PyTorch. We use MiniLM as the shared text encoder for queries and textual node attributes. The encoder produces representations of dimension $64$. Unless otherwise specified, all experiments use the same software environment and hardware configuration.

\paragraph{Local Graph Encoding and Prototype Construction.}
For each client, textual node attributes are first encoded by the shared attribute encoder $\psi$. Structural representations are obtained using a 3-layer GNN encoder $\phi$ and \texttt{RELU} activation. The attribute and structural representations are concatenated and projected into a common $d$-dimensional latent space using a MLP.

We apply $K$-means clustering independently to the fused node representations of each client. Based on the validation results reported in Section~4.4, the number of knowledge prototypes is set to $K$. Each cluster centroid is used as a client knowledge prototype and uploaded to the server during the preprocessing stage. The server constructs the global routing graph by connecting each prototype to its $L$ most similar prototypes according to cosine similarity. Self-connections are removed, and the resulting adjacency matrix is normalized before relevance propagation. 

\paragraph{Training Configuration.}
The client evidence encoder and projector are optimized using the \texttt{Adam} optimizer with an initial learning rate of \texttt{$1\times10^{-3}$}, a weight decay of \texttt{$5\times10^{-4}$}. The server routing and evidence-composition modules are optimized using a learning rate of \texttt{$1\times10^{-3}$}. We train the model for \texttt{10} epochs.

\begin{theorem}[Routing Coverage Guarantee]
\label{thm:routing_coverage}
For a query $q$, let $\mathcal C_q^+$ denote the set of evidence-bearing clients and let
$\widehat{\mathcal C}_q=\operatorname{TopB}_{i\in[M]}s_i^{(T)}(q)$
be the clients selected by the routing module. Assume that
$B\geq |\mathcal C_q^+|$. Then, the number of missed
evidence-bearing clients is bounded by
\begin{equation}
\left|\mathcal C_q^+\setminus\widehat{\mathcal C}_q\right|
\leq
\frac{\mathcal L_{\mathrm{route}}(q)}{\gamma}.
\end{equation}
Consequently, the client-level routing recall satisfies
\begin{equation}
\operatorname{Recall}_{\mathrm{client}}@B
\geq
1-\frac{\mathcal L_{\mathrm{route}}(q)}
{\gamma|\mathcal C_q^+|}.
\end{equation}
In particular, if $\mathcal L_{\mathrm{route}}(q)=0$, all
evidence-bearing clients are selected.
\end{theorem}

\begin{proof}
Let $\mathcal M_q=\mathcal C_q^+\setminus\widehat{\mathcal C}_q$
be the set of missed positive clients, with
$|\mathcal M_q|=m$. Since $B\geq|\mathcal C_q^+|$, the selected
set $\widehat{\mathcal C}_q$ contains at least $m$ negative clients.
Thus, each missed positive client $i\in\mathcal M_q$ can be matched
to a distinct selected negative client $j$.

For each missed client $i$, there exists a positive prototype
$(i,k_i)\in\mathcal P_q^+$. Let
$l_j=\arg\max_l r_{j,l}^{(T)}(q)$. Since $j$ is selected while $i$
is not selected, we have
\begin{equation}
r_{j,l_j}^{(T)}(q)
=
s_j^{(T)}(q)
\geq
s_i^{(T)}(q)
\geq
r_{i,k_i}^{(T)}(q).
\end{equation}
Therefore, the corresponding pairwise ranking term is at least
$\gamma$:
\begin{equation}
\left[\gamma-r_{i,k_i}^{(T)}(q)
+r_{j,l_j}^{(T)}(q)\right]_+
\geq \gamma.
\end{equation}
Because the $m$ matched pairs are distinct and all appear in
$\mathcal L_{\mathrm{route}}(q)$, it follows that
$\mathcal L_{\mathrm{route}}(q)\geq m\gamma$. Hence,
$m\leq\mathcal L_{\mathrm{route}}(q)/\gamma$, which proves the
result.
\end{proof}

\begin{table*}[!t]
\centering
\caption{Statistics of the datasets and knowledge graphs (KGs) used for training.}
\label{tab:dataset_statistics}
\small
\renewcommand{\arraystretch}{1.15}
\setlength{\tabcolsep}{6pt}
\begin{tabular}{lcccccc}
\toprule
Dataset & \# Query-Doc Pairs & \# Documents & \# KGs & \# Entities & \# Relations & \# Triples \\
\midrule
HotpotQA & 20,000 & 204,822 & 20 & 1,930,362 & 967,218 & 6,393,342 \\
MuSiQue  & 20,000 & 410,380 & 20 & 1,544,966 & 900,338 & 4,848,715 \\
2Wiki    & 20,000 & 122,108 & 20 & 916,907 & 372,554 & 2,883,006 \\
\midrule
\textbf{Total} & \textbf{60,000} & \textbf{737,310} & \textbf{60} &
\textbf{4,392,235} & \textbf{2,240,110} & \textbf{14,125,063} \\
\bottomrule
\end{tabular}
\end{table*}

\begin{algorithm}[t]
\caption{Initialize Procedure of FedGRAG}
\label{alg:preprocess}
\begin{algorithmic}[1]  
\Require Each graph dataset $\mathcal{G}$; Number of domain $K$; number of client $M$.
\Ensure Global prototype graph $\mathcal{R}$.

\State Initialize model parameters $\mathbf{W}$ for each client;
\State \texttt{/* On each client */}
\For{$i \gets 1$ to $M$}
\State Extract node embedding $\mathbf{H}^\psi$ and $\mathbf{H}^\phi$.
\State Obtain fusion embeddings $\mathbf{H}$.
\State Separate domain knowledge by $K$-means.
\State Generate local prototypes $\mathbf{P}$.
\State Upload local $\mathbf{P}$.
\EndFor

\State \texttt{/* On the server */}
\State Collect prototypes from each client.
\State Construct a global prototype  graph $\mathcal{R}$.

\State \Return $R$
\end{algorithmic}  
\end{algorithm}

\begin{algorithm}[t]
\caption{Query Retrieval and Generation of FedGRAG}
\label{alg:retrieval}
\begin{algorithmic}[1]
\Require Query $q$; global routing graph $\mathcal{R}$; local knowledge graphs
$\{\mathcal{G}_i\}_{i=1}^{M}$; number of selected clients $B$;
number of seed entities $N_s$; expansion depth $H$; edge budget $N_r$
\Ensure Generated answer $\hat{y}$ and composed evidence $\mathcal{E}^{q}$

\Statex \texttt{Server
query routing:}
\State Encode the query as $\mathbf{q} \gets \operatorname{Enc}(q)$.
\State Compute query-prototype similarities on $\mathcal{R}$.
\State Propagate the relevance scores over $\mathcal{R}$.
\State Compute each client score by aggregating its prototype scores
\State Select the top-$B$ clients:
$\mathcal{C}_{q} \gets \operatorname{TopB}\bigl(\{s_i(q)\}_{i=1}^{M}\bigr)$

\Statex \textbf{Local evidence retrieval and abstraction:}
\ForAll{ client $ \in \mathcal{C}_{q}$ \textbf{in parallel}}
    \State Receive query representation $\mathbf{q}$ from the server.
    \State Retrieve the top-$N_s$ seed entities from $\mathcal{G}$.
    \State Expand the seeds within $H$ hops to obtain candidate subgraph $\mathcal{S}^{q}$.
    \State Retain at most $N_r$ query-relevant edges in $\mathcal{S}^{q}$.
    \State Upload $\mathbf{H}^q,\mathbf{A}^q,\mathcal{U}^q$ to the server.
\EndFor

\Statex \textbf{Server evidence composition and generation:}
\ForAll{$i \in \mathcal{C}_{q}$}
    \State Reconstruct the evidence subgraph:
$\hat{\mathcal{G}}_i^q
\gets D_{\phi}(\mathcal{P}_i^q)$;
\State Convert $\hat{\mathcal{G}}_i^q$ into evidence facts $\hat{\mathcal{B}}_i^q$.
\EndFor
\State Aggregate the reconstructed facts:
$\hat{\mathcal{B}}^{q}
\gets \cup_{i\in\mathcal{C}_{q}}\widehat{\mathcal{B}}_{i}^{q}$
\State Construct the evidence interaction graph
$\mathcal{G}_{E}^{q}$ from $\widehat{\mathcal{B}}^{q}$.
\State Traverse $\mathcal{G}_{E}^{q}$ and serialize the selected evidence paths into $\mathcal{E}^{q}$.
\State Generate the answer:
$\hat{y}\gets\operatorname{LLM}\bigl(\operatorname{Prompt}(q,\mathcal{E}^{q})\bigr)$
\State \Return $\hat{y},\mathcal{E}^{q}$.
\end{algorithmic}
\end{algorithm}

\end{document}